\newif\ifanon\anonfalse 

\documentclass[11pt]{article}
\usepackage[margin=1in]{geometry}
\usepackage[utf8]{inputenc}
\usepackage[T1]{fontenc} 

\usepackage{mathpazo}
\usepackage{framed}

\usepackage{xcolor}
\usepackage{color}
\definecolor{niceRed}{RGB}{190,38,38}
\definecolor{niceYellow}{HTML}{f5b400}
\definecolor{blueGrotto}{HTML}{059DC0}
\definecolor{royalBlue}{HTML}{057DCD}
\definecolor{navyBlue}{HTML}{0B579C}
\definecolor{limeGreen}{HTML}{81B622}
\definecolor{nicePurple}{HTML}{9c27b0}
\definecolor{lightRoyalBlue}{HTML}{def2ff}  
\definecolor{gold}{HTML}{ffa300}

\newcommand{\white}[1]{\textcolor{white}{#1}}

\usepackage{color-edits} 
\addauthor[Jane]{jl}{blue}
\addauthor[Anay]{am}{gold}
\addauthor[GPT]{gpt}{niceRed}
\addauthor[Manolis]{mz}{teal}

\usepackage{hyperref}
\hypersetup{
  colorlinks = true, 
  urlcolor = {blueGrotto},
  linkcolor = {royalBlue},
  citecolor = {navyBlue}
}
\usepackage[
  natbib=true,
  style=alphabetic,
  sorting=alphabeticlabel,
  sortcites=true,
  minbibnames=3,
  maxbibnames=999,
  mincitenames=1,
  maxcitenames=4,
  minalphanames=1,
  maxalphanames=6,
]{biblatex}

\DeclareSortingTemplate{alphabeticlabel}{
  \sort[final]{%
    \field{labelalpha} 
  }
  \sort{%
    \field{year}   
  }
  \sort{%
    \field{title}  
  }
}

\AtBeginRefsection{\GenRefcontextData{sorting=ynt}}
\AtEveryCite{\localrefcontext[sorting=ynt]}
\AtEveryBibitem{%
  \clearname{editor}%
}
\AtEveryBibitem{%
  \clearlist{location}%
} 

\usepackage{booktabs} 
\usepackage{multicol} 
\usepackage{multirow} 
\usepackage{longtable}

\usepackage{subcaption}
\usepackage{graphicx}
\usepackage{float} 
\usepackage{tikz}
\usepackage{pgfplots}
\pgfplotsset{compat=1.17}
\usepgfplotslibrary{fillbetween}
\usetikzlibrary{arrows.meta}
\usetikzlibrary{calc}

\tikzset{
  myNodeFlex/.style={
    draw,
    rectangle,
    rounded corners,
    text centered,
    minimum height=1.5em,
  }
}

\tikzset{
  myNode/.style={
    draw,
    rectangle,
    rounded corners,
    text centered,
    minimum height=1.5em,
    minimum width=3cm,
    text width=5cm,    
  }
}

\tikzset{
  myNodeNarrow/.style={
    draw,
    rectangle,
    rounded corners,
    text centered,
    minimum height=1.5em,
    minimum width=1cm,
  }
}

\tikzset{
  myNodeWide/.style={
    draw,
    rectangle,
    rounded corners,
    text centered,
    minimum height=1.5em,
    minimum width=6cm,
  }
}

\usepackage{physics} 
\usepackage{amsmath}
\usepackage{amssymb}
\usepackage{amsthm}
\usepackage{bbm}
\usepackage{blkarray}
\usepackage{mathtools}
\usepackage{nicefrac}
\usepackage{dsfont}
\usepackage{pifont} 
\usepackage{thm-restate} 

\usepackage{setspace}

\usepackage{typed-checklist}
\usepackage{enumerate} 
\usepackage{enumitem} 
\usepackage{tcolorbox}
\usepackage[framemethod=TikZ]{mdframed}
\mdfsetup{%
backgroundcolor=royalBlue!0, 
roundcorner=4pt,
skipabove=5pt,
skipbelow=5pt,
linewidth=1pt}
\usepackage{changepage} 

\usepackage{algorithm}
\usepackage{algpseudocode}[1]

\usepackage{mathrsfs} 
\usepackage{soul} 
\setuldepth{Berlin}

\usepackage{etoolbox}
\usepackage{array}
\usepackage[capitalise,noabbrev,nameinlink,sort]{cleveref} 

\theoremstyle{plain} 
\newtheorem{theorem}{Theorem}

\newtheorem{lemma}[theorem]{Lemma}

\newtheorem{definition}{Definition}

\newtheorem*{definition*}{Definition}

\theoremstyle{definition}

\theoremstyle{remark}

\AfterEndEnvironment{definition}{\noindent\ignorespaces}
\AfterEndEnvironment{infdefinition}{\noindent\ignorespaces}
\AfterEndEnvironment{assumption}{\noindent\ignorespaces}
\AfterEndEnvironment{problem}{\noindent\ignorespaces}
\AfterEndEnvironment{openproblem}{\noindent\ignorespaces}
\AfterEndEnvironment{lemma}{\noindent\ignorespaces}
\AfterEndEnvironment{theorem}{\noindent\ignorespaces}
\AfterEndEnvironment{proposition}{\noindent\ignorespaces}
\AfterEndEnvironment{fact}{\noindent\ignorespaces}
\AfterEndEnvironment{question}{\noindent\ignorespaces}
\AfterEndEnvironment{corollary}{\noindent\ignorespaces}
\AfterEndEnvironment{model}{\noindent\ignorespaces}
\AfterEndEnvironment{remark}{\noindent\ignorespaces}
\AfterEndEnvironment{proof}{\noindent\ignorespaces}
\AfterEndEnvironment{fact}{\noindent\ignorespaces}
\AfterEndEnvironment{minftheorem}{\noindent\ignorespaces}
\AfterEndEnvironment{inftheorem}{\noindent\ignorespaces}
\AfterEndEnvironment{maintheorem}{\noindent\ignorespaces}
\AfterEndEnvironment{restatable}{\noindent\ignorespaces}
\AfterEndEnvironment{infassumption}
{\noindent\ignorespaces}
\AfterEndEnvironment{infcorollary}{\noindent\ignorespaces}

\crefname{section}{Section}{Sections}
\crefname{theorem}{Theorem}{Theorems}
\crefname{lemma}{Lemma}{Lemmas}
\crefname{problem}{Problem}{Problems}
\crefname{program}{Program}{Programs}
\crefname{definition}{Definition}{Definitions}
\crefname{conjecture}{Conjecture}{Conjectures}
\crefname{corollary}{Corollary}{Corollaries}
\crefname{construction}{Construction}{Constructions}
\crefname{conjecture}{Conjecture}{Conjectures}
\crefname{claim}{Claim}{Claims}
\crefname{observation}{Observation}{Observations}
\crefname{proposition}{Proposition}{Propositions}
\crefname{fact}{Fact}{Facts}
\crefname{question}{Question}{Questions}
\crefname{problem}{Problem}{Problems}
\crefname{remark}{Remark}{Remarks}
\crefname{example}{Example}{Examples}
\crefname{equation}{Equation}{Equations}
\crefname{appendix}{Section}{Sections}
\crefname{algorithm}{Algorithm}{Algorithms}
\crefname{model}{Model}{Models}
\crefname{figure}{Figure}{Figures}
\crefname{infassumption}{Informal Assumption}{Informal Assumptions}
\crefname{inftheorem}{Informal Theorem}{Informal Theorems}
\crefname{infdefinition}{Informal Definition}{Informal Definitions}
\crefname{minftheorem}{Main Informal Theorem}{Main Informal Theorems}
\crefname{maintheorem}{Main Theorem}{Main Theorems}
\crefname{assumption}{Assumption}{Assumptions}
\crefname{step}{Step}{Steps}
\crefname{result}{Result}{Results}
\crefname{event}{Event}{Events}
\crefname{none}{}{}

\newlist{asmpenum}{enumerate}{1} 
\setlist[asmpenum]{label={\arabic*.},ref=\theassumption.{\arabic*}}
\crefname{asmpenumi}{Assumption}{Assumptions}

\BeforeBeginEnvironment{subenvironment}{\white{.}\\ \vspace{-10mm}}
\AfterEndEnvironment{subenvironment}{\vspace{4mm}}

\newcommand{\yesnum}{\refstepcounter{equation}\tag{\theequation}}
\makeatletter
\newcommand{\tagnum}[2]{%
    \refstepcounter{equation}%
    \tag{#1) \ (\theequation}%
    \protected@write \@auxout {}{%
        \string \newlabel {#2}{{\theequation}{\thepage}{}{equation.\theequation}{}}%
    }%
}
\makeatother

\makeatletter
\renewcommand{\eqref}[1]{\textup{\eqrefform@{\ref{#1}}}}
\let\eqrefform@\tagform@

\newcommand{\changetag}[1]{%
  \renewcommand\tagform@[1]{\maketag@@@{(\ignorespaces#1\unskip\@@italiccorr)}}%
}
\makeatother

\newcommand{\qquadtext}[1]{\qquad\text{#1}\qquad}

\newcommand{\qquadand}{\qquadtext{and}}

\def\abs#1{\left| #1 \right|}

\newcommand{\inbrace}[1]{\left\{#1\right\}}

\newcommand{\inparen}[1]{\left(#1\right)}
\newcommand{\insquare}[1]{\left[#1\right]}

\let\norm\relax
\newcommand{\norm}[1]{\ensuremath{\left\| #1\right\|}}

\newcommand{\R}{\mathbb{R}}

\newcommand{\Z}{\mathbb{Z}}

\newcommand{\E}{\operatornamewithlimits{\mathbb{E}}} 
\newcommand{\Ex}{\E}

\newcommand{\sfrac}[2]{{#1/#2}} 
\newcommand{\nfrac}[2]{\nicefrac{#1}{#2}}

\newcommand{\eps}{\varepsilon}
\renewcommand{\epsilon}{\varepsilon}
\makeatletter
\newcommand*{\tran}{{\mathpalette\@tran{}}}
\newcommand*{\@tran}[2]{\raisebox{\depth}{$\m@th#1\intercal$}}
\makeatother

\mathchardef\NABLA"272
\newcommand*{\Nabla}{\boldsymbol\NABLA}
\let\nabla\Nabla

\newcommand{\wt}[1]{\widetilde{#1}}

\DeclareMathAlphabet{\mathdutchcal}{U}{dutchcal}{m}{n}
\SetMathAlphabet{\mathdutchcal}{bold}{U}{dutchcal}{b}{n}
\DeclareMathAlphabet{\mathdutchbcal}{U}{dutchcal}{b}{n}

\DeclareMathAlphabet\urwscr{U}{urwchancal}{b}{n}%
\DeclareMathAlphabet\rsfscr{U}{rsfso}{m}{n}
\DeclareMathAlphabet\euscr{U}{eus}{m}{n}
\DeclareFontEncoding{LS2}{}{}
\DeclareFontSubstitution{LS2}{stix}{m}{n}
\DeclareMathAlphabet\stixcal{LS2}{stixcal}{m} {n}

\renewcommand{\paragraph}[1]{\medskip \noindent\textbf{#1}~}

\newcommand{\ie}{\textit{i.e.}}
\newcommand{\eg}{\textit{e.g.}}

\newcommand{\hypo}[1]{\mathdutchcal{#1}}

\newcommand{\hyH}{\hypo{H}}

\newcolumntype{L}[1]{>{\raggedright\let\newline\\\arraybackslash\hspace{0pt}}m{#1}}
\newcolumntype{C}[1]{>{\centering\let\newline\\\arraybackslash\hspace{0pt}}m{#1}}
\newcolumntype{R}[1]{>{\raggedleft\let\newline\\\arraybackslash\hspace{0pt}}m{#1}}

\title{A Note on Non-Negative $L_1$-Approximating Polynomials}

\author{ 
        \centering
        \hspace{4mm}
        \begin{tabular}{C{4.75cm}C{4.75cm}C{4.75cm}}
        {\bf Jane H. Lee} & {\bf Anay Mehrotra} & {\bf Manolis Zampetakis}\\
        Yale University & Stanford University & Yale University\\[-6mm]
        {\small\phantom{....................}} \mbox{\small\href{mailto:jane.h.lee@yale.edu}{\texttt{jane.h.lee@yale.edu}}} & {\small \phantom{............}} \mbox{\small\href{mailto:anaymehrotra1@gmail.com}{\texttt{anaymehrotra1@gmail.com}}} & \mbox{\small\href{mailto:manolis.zampetakis@yale.edu}{\texttt{manolis.zampetakis@yale.edu}}}
        \\[-3mm]
        \end{tabular} 
} 

\date{}

\begin{document}

\maketitle
\thispagestyle{empty}

\begin{abstract}
    $L_1$-Approximating polynomials, \ie{}, polynomials that approximate indicator functions in $L_1$-norm under certain distributions, are widely used in computational learning theory.
    We study the existence of \textit{non-negative} $L_1$-approximating polynomials with respect to Gaussian distributions.
    This is a stronger requirement than $L_1$-approximation but weaker than sandwiching polynomials (which themselves have many applications). 
    These non-negative approximating polynomials have recently found uses in smoothed learning from positive-only examples. 

    In this short note, we prove that every class of sets with Gaussian surface area (GSA) at most $\Gamma$ under the standard Gaussian admits degree-$k$ non-negative polynomials that $\eps$-approximate its indicator functions in $L_1$-norm, for
    $k=\wt{O}\!\inparen{\sfrac{\Gamma^2}{\eps^2}}$.
    Equivalently, finite GSA implies $L_1$-approximation with the stronger pointwise guarantee that the approximating polynomial has range contained in $[0,\infty)$.
    Up to a constant-factor, this matches the degree of the best currently known Gaussian $L_1$-approximation degree bound without the non-negativity constraint. 
\end{abstract}

\section{Introduction and Main Result}
Low-degree polynomial approximation is a central tool in learning theory \cite{linial1993constant}.
It underlies the classical $L_1$-Regression algorithm for agnostic learning \cite{kalai2008agnostically,klivans2008gaussian} and, more recently, has played an important role in testable learning and related distributional learning problems \cite{rubinfeld2023testing,gollakota2023momentMatching,klivans2023testable,rubinfeld2026safely}.
A particularly strong notion that has emerged in this latter literature is that of \emph{sandwiching polynomials}, where one asks for polynomials $p_{\mathrm{down}}$ and $p_{\mathrm{up}}$ satisfying $p_{\mathrm{down}}(x)\leq \mathds{1}_H(x)\leq p_{\mathrm{up}}(x)$ for all $x$ together with a small expected gap between the two \cite{ksv2026sandwiching}.
Here, the small separation between $p_{\mathrm{down}}$ and $p_{\mathrm{up}}$ together with the ``sandwiching constraint'' implies that both $p_{\mathrm{down}}$ and $p_{\mathrm{up}}$ are $L_1$-approximating \mbox{polynomials for $\mathds{1}_H$.
We study a weaker notion:}
\vspace{-6mm}
\begin{definition}[Non-Negative $L_1$-Approximation]
    Fix a distribution $\mu$ on $\R^d$ and measurable $H \subseteq \R^d$.
    Given $\eps>0$, a polynomial $q$ is said to be \emph{non-negative $L_1$-approximating polynomial} for $H$ under $\mu$ if
    \vspace{-1mm}
    \[
        \Ex_{x\sim \mu}\insquare{\abs{q(x)-\mathds{1}_H(x)}} \leq \eps
        \qquadand
        \mathrm{Range}(q)\subseteq [0,\infty)\,. 
    \]    
    \vspace{-7mm}
\end{definition}
More generally, one could ask only that the range of $q$ be bounded below by a constant $-C$. 
We focus on the strongest requirement of $C=0$ here.
Observe that every upper sandwiching polynomial $p_{\rm up}$ is automatically a non-negative $L_1$-approximant to the indicator function, so sandwiching implies the existence of non-negative approximants.
The converse, however, need not hold:
indeed, a non-negative $L_1$-approximating polynomial may not satisfy $q(x)\geq \mathds{1}_H(x)$ pointwise.  
Hence, non-negative $L_1$ approximation sits in between ordinary $L_1$-approximation and the stronger pointwise notion of sandwiching.

This weaker notion already appears to be useful:
In particular, it is sufficient for recent constrained polynomial-regression approaches to smoothed positive-only learning, where one needs a polynomial approximation and the above one-sided range requirement, but does not require a full sandwiching pair \cite{lee2026smoothed}.
Further, similar (but different) notions have also been useful for reliable learning; see \cref{sec:discussion}.
\textit{All of this motivates isolating non-negative $L_1$-approximation as an object that may be of independent interest.}

\smallskip

\noindent
Our main observation is that, in the Gaussian setting, the existing $L_1$-approximation machinery already yields this weaker notion at essentially no loss.
More precisely, finite \textit{Gaussian Surface Area} (GSA) is sufficient to guarantee the existence of non-negative approximating polynomials of degree that matches the best known unconstrained Gaussian $L_1$-approximation bounds up to a constant factor.
We prove the following result.
\begin{theorem}[Non-negative $L_1$-approximation for GSA Classes]
\label{thm:non-negative-approx}
Let $\hyH$ be a class of Borel subsets of $\R^d$ and define
$\Gamma_{\hyH} \coloneqq \sup_{H\in \hyH} \Gamma(H),$
where $\Gamma(H)$ denotes the GSA of $H$. Suppose that $\Gamma_{\hyH} < \infty$. Then, for every $H\in \hyH$ and every $\eps \in \inparen{0,\nfrac{1}{2}}$, there is a polynomial $q_H\colon \R^d \to \R_{\geq 0}$ such that 
\[
    \Ex_{x\sim \gamma_d}\insquare{\abs{q_H(x)-\mathds{1}_H(x)}} \leq \eps
    \qquadand
    \deg\!\inparen{q_H}
    \leq     O\!\inparen{\frac{\Gamma_{\hyH}^2}{\eps^2}\,\log{\frac1\eps}}.
\]
In particular, note that each $q_H$ has range contained in $[0,\infty)$.
\end{theorem}
Thus, any Gaussian concept class with uniformly bounded GSA automatically admits non-negative $L_1$ approximating polynomials.
Moreover, \cref{thm:non-negative-approx} directly upgrades $L_1$ approximating polynomials constructed for GSA classes by \cite{pesenti2026agnostic,klivans2008gaussian} to \emph{non-negative} Gaussian $L_1$-approximation in the following sense:
    if $p_H$ is the low-degree polynomial produced by the truncated Hermite-expansion of the Gaussian smoothing of $2\, \mathds{1}_H-1$, then $q_H=(\nfrac{1}{4})(1+p_H(x))^2$ is a non-negative $L_1$-approximating polynomial for $\mathds{1}_H$.
Hence, \cref{thm:non-negative-approx} shows that to ensure non-negativity we only require an additional factor of $2$ over the degree bounds obtained by \cite{pesenti2026agnostic,klivans2008gaussian}.

Specializing this result with the classical bounds on the GSA of $O\!\inparen{\sqrt{\log m}}$ for intersections of $m$ halfspaces \cite{klivans2008gaussian} and $O\!\inparen{d^{1/4}}$ for convex sets \cite{klivans2008gaussian,ball1993gsaConvexSets} yields degree bounds of $O\!\inparen{
        \frac{\log m}{\eps^2}
         \, \log{\frac1\eps}
     }$ and $O\!\inparen{
        \frac{\sqrt{d}}{\eps^2}\, \log{\frac1\eps}
    }$ respectively.

\begin{proof}[Proof of \cref{thm:non-negative-approx}]
Fix $H\in \hyH$ and write
\[
h \coloneqq \mathds{1}_H
\qquadand
f \coloneqq 2h-1\,.
\]
As is standard, we construct a polynomial approximation to $h$ by first smoothing $f$ and taking its truncated Hermite expansion.
To obtain a non-negative polynomial approximation we do an additional post-processing step.
For $\rho \in \insquare{0,1}$, let $T_\rho$ denote the Ornstein--Uhlenbeck operator,
\[
T_\rho f(x) \coloneqq \Ex_{Z\sim \gamma_d}\insquare{f\!\inparen{\rho x + \sqrt{1-\rho^2}\,Z}}\,,
\]
and, for $t\in \Z_{\geq 0}$, let $\Pi_{\leq t}$ denote truncation of the Hermite expansion to degree at most $t$.
Define 
\[
    g\coloneqq T_\rho f\,,\qquad 
    p\coloneqq \Pi_{\leq t} g\,,\qquadand
    q = \frac{1}{4}(1+p)^2\,.
\]
Observe that $q(x)\geq 0$ for every $x$, and $\deg\!\inparen{q}\leq 2t$. 
Next, we record some standard facts.

\begin{lemma}[\citep{pesenti2026agnostic,klivans2008gaussian}]
\label{lem:gsa-smoothing-input}
Let $\gamma_d$ be the standard Gaussian measure on $\R^d$. 
For $\rho\in\insquare{0,1}$, 
\[
    \|f-g\|_{L_1(\gamma_d)} \leq 2\sqrt{\pi}\,\Gamma(H)\sqrt{1-\rho}
    \qquadand
    \|g-p\|_{L_2(\gamma_d)} \leq \rho^{t+1}\,.
\]

\end{lemma}
Our goal is to bound $\norm{q-h}_{L_1(\gamma_d)}$.
Toward this, observe that 
\[
    \norm{q-h}_{L_1(\gamma_d)}
    = 
    \norm{q-\frac{f+1}{2}}_{L_1(\gamma_d)}
    \leq 
    \norm{q-\frac{g+1}{2}}_{L_1(\gamma_d)}
    + \frac{1}{2} \norm{f-g}_{L_1(\gamma_d)}\,.
    \yesnum\label{eq:target}
\]
Substituting the value of $q$ in $q-\frac{g+1}{2}$, we obtain that 
\[
    q-\frac{g+1}{2}
    = \frac{p^2+2p+1-2g-2}{4}
    =
    \frac{(p-g)(p+g+2) + (g^2-1)}{4}\,.
\]
Combining this with the triangle inequality implies that
\[
    \norm{q-\frac{g+1}{2}}_{L_1(\gamma_d)}
    \leq 
    \frac{1}{4}\norm{(p-g)(p+g+2)}_{L_1(\gamma_d)}
    + 
    \frac{1}{4}\norm{g^2-1}_{L_1(\gamma_d)}\,.
    \yesnum\label{eq:decomposition}
\]
Since $g$ is an average of $\pm 1$-random variables, it satisfies $g(\cdot)\in\insquare{-1,1}$. 
Also, since $\Pi_{\leq t}$ is an orthogonal projection in $L_2(\gamma_d)$ and $T_\rho$ is an $L_2(\gamma_d)$-contraction (\eg{}, \cite[Lemma A.2]{pesenti2026agnostic}),  
\[
\norm{p}_{L_2(\gamma_d)} \leq \norm{g}_{L_2(\gamma_d)} \leq \norm{f}_{L_2(\gamma_d)} = 1\,.
\]
Together these imply that
\[
    \norm{p+g+2}_{L_2(\gamma_d)}
    \leq \norm{p}_{L_2(\gamma_d)} + \norm{g}_{L_2(\gamma_d)} + 2
    \leq 4\,.
    \yesnum\label{eq:term1}
\]
Further, since $f(x)\in \inbrace{\pm 1}$, $f(x)^2=1$ for any $x$ and, hence,
$\norm{g^2-1}_{L_1(\gamma_d)} 
    = \norm{g^2-f^2}_{L_1(\gamma_d)} 
    = \norm{(f-g)(f+g)}_{L_1(\gamma_d)}.$
Now, as observed before $\abs{f(x)}\leq 1$ and $\abs{g(x)}\leq 1$ for each $x$, implying: 
\vspace{-1mm}
\[
    \norm{g^2-1}_{L_1(\gamma_d)} 
    \leq 2\norm{f-g}_{L_1(\gamma_d)}\,.
    \yesnum\label{eq:term2}
\]
\vspace{-1mm}
Substituting \cref{eq:term1,eq:term2} into \cref{eq:decomposition} implies 
\vspace{-1mm}
\begin{align*}
    \norm{q-\frac{g+1}{2}}_{L_1(\gamma_d)}
    ~~
    &\leq~~ 
    \frac{1}{4}\norm{p-g}_{L_2(\gamma_d)}\cdot \norm{p+g+2}_{L_2(\gamma_d)}
    + 
    \frac{1}{2}\norm{f-g}_{L_1(\gamma_d)}\\
    &\leq~~ \norm{p-g}_{L_2(\gamma_d)}
    + 
    \frac{1}{2}\norm{f-g}_{L_1(\gamma_d)}\,.
\end{align*}
\vspace{-1mm}
Combining this with \cref{eq:target}, yields
\[
    \norm{q-h}_{L_1(\gamma_d)}
    \leq 
    \norm{p-g}_{L_2(\gamma_d)}
    + 
    \norm{f-g}_{L_1(\gamma_d)}\,.
\]
Applying \cref{lem:gsa-smoothing-input}, we obtain
\[
\norm{q-h}_{L_1(\gamma_d)}
\leq \rho^{t+1} + 2\sqrt{\pi}\,\Gamma(H)\sqrt{1-\rho}\,.
\]
It remains to choose $\rho$ and $t$. 
If $\Gamma(H)=0$, then taking $\rho=0$ and $t=0$ gives a constant polynomial $q$ with zero error. 
Hence, it remains to consider the case $\Gamma(H)>0$.
Set
\[
    \rho \coloneqq 1-\min\!\inbrace{1,~\frac{\eps^2}{16\pi\,\Gamma(H)^2}}\,.
\]
So that $2\sqrt{\pi}\,\Gamma(H)\sqrt{1-\rho} \leq \nfrac{\eps}{2}.$
If $\rho=0$, take $t=0$, and the conclusion follows. Otherwise, choose the smallest $t$ so that
$\rho^{t+1} \leq \nfrac{\eps}{2}.$
Since $\log{\nfrac1\rho} \geq 1-\rho$ for $\rho\in\inparen{0,1}$, it is enough to require
$t+1 \geq \frac{1}{1-\rho}\cdot \log{\nfrac2\eps}.$
For this choice,
\[
\norm{q-h}_{L_1(\gamma_d)} \leq \eps\,.
\]
Moreover,
$t
=
O\!\inparen{\frac{\Gamma(H)^2}{\eps^2}\,\log{\frac1\eps}},$
and, therefore, as desired, $\deg\!\inparen{q}
\leq 2t
=
O\!\inparen{\frac{\Gamma(H)^2}{\eps^2}\,\log{\frac1\eps}}.$
\end{proof}

\section{Discussion and Related Work}
\label{sec:discussion}
\cref{thm:non-negative-approx} shows that, in the Gaussian setting, lower-bounded $L_1$-approximation is essentially no harder than ordinary $L_1$-approximation:
finite GSA already implies the existence of non-negative $L_1$-approximating polynomials of essentially the same degree.
In contrast, a similarly general implication is not known for sandwiching polynomials over arbitrary finite-GSA classes.
That said, recent work has obtained strong sandwiching bounds for several important families with low intrinsic dimension, which include some, but \mbox{not all, finite-GSA classes (at least provably) \cite{ksv2026sandwiching}.}

Another related notion appears in reliable learning.
\citet{kkm12reliable} introduced reliable agnostic learning, and \citet{kt14reliable} connected it to one-sided approximating polynomials.
For a Boolean function $f\colon \{-1,1\}^n \to \{-1,1\}$, a positive one-sided $\eps$-approximating polynomial $p$ is required to satisfy
\[
p(x) \in [1-\eps,\infty)
\quad \text{when } f(x)=1\,,
\qquadand
p(x) \in [-1-\eps,-1+\eps]
\quad \text{when } f(x)=-1\,,
\]
with the negative one-sided notion defined symmetrically.
After the affine rescaling $q = (p+1)/2$, such a polynomial is bounded below by $-\nfrac{\eps}{2}$ on the Boolean cube.
Thus, one-sided approximation is related to the lower-bounded $L_1$ notion studied here, but the two notions are incomparable:
one-sided approximation imposes asymmetric pointwise control, whereas our notion is distributional and average-case.
To the best of our knowledge, no analogue of \cref{thm:non-negative-approx} is known for one-sided approximating polynomials.
The reliable-learning literature does, however, provide several concrete existence results; see, for instance, the works of \cite{kt14reliable,kkm12reliable}.

More broadly, it would be interesting to understand whether similarly generic reductions exist beyond the Gaussian setting.
In particular, it would be interesting to understand when and which $L_1$-approximation can be upgraded to lower-bounded $L_1$-approximation with little or no loss in degree, and how far this weaker notion remains from full sandwiching.
We view clarifying the gap between (ordinary) $L_1$-approximation, lower-bounded $L_1$-approximation, and sandwiching $L_1$-approximation as an interesting direction for future work.

\medskip
\noindent\textbf{Acknowledgments.}
    We thank Kostas Stavropoulos for helpful discussions on the literature and feedback on a draft of this note.
    
\printbibliography

\end{document}